\documentclass[letterpaper]{article}
\usepackage{proceed2e}
\usepackage[margin=1in]{geometry}
\usepackage{amsmath}
\usepackage{amsthm}
\usepackage{graphicx}
\usepackage{amssymb}
\usepackage{hyperref}

\DeclareMathOperator*{\argmin}{\arg\!\min}
\DeclareMathOperator*{\argmax}{\arg\!\max}
\DeclareMathOperator*{\UNK}{UNK}

\newtheorem*{theorem}{Theorem}
\newtheorem*{lemma}{Lemma}

\usepackage{times}
\usepackage[]{algorithm2e}
\usepackage{algpseudocode}
\SetKwInput{kwRequire}{Require}
\SetKwInput{kwReturn}{return}
\SetKwInput{kwWhere}{where}

\title{Learning to Acquire Information}

\author{} 

\author{ {\bf Yewen Pu} \\
MIT \\
yewenpu@mit.edu \\
\And
{\bf Leslie P Kaelbling}  \\
MIT          \\
lpk@csail.mit.edu \\
\And
{\bf Armando Solar-Lezama}   \\
MIT \\
asolar@csail.mit.edu   \\
}

\begin{document}

\maketitle

\begin{abstract}
We consider the problem of diagnosis where a set of simple observations are used to infer a potentially complex hidden hypothesis. Finding the optimal subset of observations is intractable in general, thus we focus on the problem of active diagnosis, where the agent selects the next most-informative observation based on the results of previous observations. We show that under the assumption of uniform observation entropy, one can build an \emph{implication model} which directly predicts the outcome of the potential next observation conditioned on the results of past observations, and selects the observation with the maximum entropy. This approach enjoys reduced computation complexity by bypassing the complicated hypothesis space, and can be trained on observation data alone, learning how to query without knowledge of the hidden hypothesis.
\end{abstract}

\section{INTRODUCTION}

In active diagnosis, an agent attempts to discover a hidden hypothesis by asking a series of well chosen questions. For instance, a sushi shop might wish to learn the preference of a customer by asking a series comparison questions; a network monitor might wish to detect a faulty link by making a sequence of connectivity queries.

The difficulty of diagnosis lies in the cost of making observations: it is often too expensive to obtain results for all the observations. A more practical task is to find a subset of observations that yields the most information. Krause and Guestrin (2011) demonstrated that a greedy strategy that iteratively maximizes information gain can obtain a subset of observations that is near-optimal; Golovin and Krause (2011) extended the result to the adaptive case, where the agent chooses the next most-informative observation based on the outcomes of all previous observations. In this paper we consider the setting where the agent is using the greedy strategy in the adaptive case.

However, even in the greedy scheme there can be a significant computational cost. At execution time, the information gain must be computed for each candidate observation, which requires computing entropy terms that sum or integrate over the complex hypothesis space, which can be computationally infeasible. For example, in the case of detecting link failures, the total number of possible failure configurations is $2^M$ where $M$ is the number of links. Prior work has developed various approximation methods to make information gain computable. Rish (2004), Zhen (2012) and Bellala (2013) simplify the complex hypothesis space by exploiting independence structures in the graphical model and using problem specific assumptions to further restrict the hypothesis space. In cases where observations can rule out a hypothesis completely, one can maintain a version space: a set of hypotheses that are consistent with the observations made so far, and attempt to shrink its size through observations. Tong (2001) and Nowak (2008) directly reduce the size of the version space. Seung (1992) maintains a sample of the version space, called a committee, and selects the observation that maximally reduces the size of the committee. Joshi (2009), Tong (2001), and Lewis (1994) consider the special case where the hypothesis is a classifier and observations are data points. They employ a heuristic that maintains a single current-best classifier, and selects the datapoint that is closest to its decision boundary. Holub (2008) aims to reduce the entropy on the entire dataset, and uses a committee of classifiers to estimate this entropy.

We propose an approach to active diagnosis that is fundamentally different from these previous approaches: We wish to select the next most-informative observation without explicitly modeling or maintaining any notion of the hypothesis space, but by modeling the implications between observations directly. The intuition is that some observations might be redundant given other observations: In the network example if node $i$ is connected to node $j$ via a single path, connectivity between them implies all links on the path are functional. In general, implications may be probabilistic rather than deterministic: Given $K$ current observations $\{O^1=o^1 \dots O^K=o^k\}$, one can model the outcome of an unseen observation $O'$ by $P(O'|O^1=o^1 \dots O^K=o^k)$. Our key result is that under the assumption of uniform observation entropy (given a particular hypothesis), the next most-informative observation is precisely the observation with the maximum entropy conditioned on the observations already made:
\begin{equation*}
\begin{split}
  O^{best} = \argmax_{O'} H(O'|O^1=o^1 \dots O^K=o^k)
\end{split}
\end{equation*}
This entropy can be easily computed from $P(O'|O^1=o^1 \dots O^K=o^k)$, which we model directly with a neural-network called the implication model. It can be trained in a supervised fashion on a  dataset of observation data: The dataset with $n$ training examples will contain values for all possible observation dimensions associated with each of n underlying true hypotheses, but does not require that it be labeled with that true hypothesis. Once the implication model is trained, it is used in an observation collection algorithm, OC, which adaptively selects the next most-informative observation conditioned on previous observations. Once a set of informative observations are made, we deliver the best hypothesis from these observations in a problem specific manner.

This work makes the following contributions:
\begin{itemize}
	\item We prove that under the assumption of uniform observation entropy, the next most-informative observation can be made without explicit modeling of the hypothesis space, but by modeling the implications between observations instead (Section 2). This model has the dual advantages of being computationally efficient to use and being trainable on unlabeled data.
    \item We describe a method of modeling and training the implication model using a neural network, using techniques from sparsifying auto-encoders (Section 3).
    \item We evaluate our approach on several distinct benchmark problems, and provide problem-specific algorithms for delivering the best hypothesis after a set of informative observations have been made. We compare our results against reasonable baseline algorithms (Section 4).
\end{itemize}

\section{OVERVIEW}

To give an intuition of our approach, consider a simplified, one-player variant of the game BattleShip \footnote{For rules of this game see: \url{https://en.wikipedia.org/wiki/Battleship_(game)}}. Rather than ``sinking'' all the ships on a given board, the task is to infer the structure of the board itself, which consist of locations and orientations of all the ships. 


One algorithm to solve this problem is to keep a set of all possible boards, and with each additional observation, discard all the boards which are incompatible with that observation (for instance, if a query to coordinate $(3,4)$ results in a ``miss'' then all the boards with a ship occupying the coordinate $(3,4)$ will be discarded). A query is chosen by maximizing the ``disagreement'' between all the remaining boards in the set (Nowak 2008). Intuitively, if a coordinate $(x,y)$ is occupied by a ship in half of the remaining board and unoccupied in the other half (high disagreement), a query to this coordinate would narrow down the choice of compatible boards by half. This algorithm terminates when there is only 1 board remaining in the set.

The drawback of the aforementioned algorithm is that the computation of disagreement, which requires enumeration over all the remaining boards in the set. In a game such as battleship, there can be an exponentially many number of boards (up to $2^{38}$ in our experiment), which makes the first query already infeasible to compute. 

Instead, we assume the existence of an oracle, which, given a set of past observations (a set of coordinates and whether it resulted in a ``hit'' or a ``miss''), is capable of predicting the probability of ``hit'' or ``miss'' of an arbitrary unseen coordinate. Given this oracle, one can ask the following question: Which unseen coordinate is the most ``confusing'' to the oracle (i.e. closest to a $50\% - 50\%$ split between ``hit'' and ``miss'')? It turns out that under the assumption of uniform observation entropy, the most informative query coincides with the query which maximally confuses the oracle. The existence of such an oracle allows one to make the most informative query without having to enumerate over the space of possible boards, saving computation time. We model the oracle as a neural network, which is trained offline on the task of board-completion: Given a complete board $B$ where all the coordinates are labeled as either ``hit'' or ``miss'', a random subset of these coordinates are obscured by hiding the labels, forming a partially observable board $B'$. The neural network is tasked to produce $B$ from $B'$ in via supervised learning.

We will now formalize the specifics of this approach.   
\section{ACTIVE DIAGNOSIS}

In this section we formally define the problem of active diagnosis, show that under the uniform observation entropy assumption the next most-informative test is the conditioned observation entropy maximizer, and present our algorithm for active diagnosis.

\subsection{PROBLEM DEFINITION}
An observation problem is made up of a hidden hypothesis $S$, which is a random variable that can take on values $s \in dom(S)$, and a finite set of observation random variables $O_1 \dots O_N$, which can take on values $o_i \in dom(O_i)$. The hypothesis is drawn from a distribution $s \sim P(S)$ and given a hypothesis the observations are drawn from conditional distributions $\forall i, o_i \sim P(O_i | S=s)$. In this work we consider a discrete hypothesis space and observations with discrete domains.

We adopt the greedy strategy proposed by Golovin and Krause (2011) that aims to maximize information gain one observation at a time: Given a (potentially empty) set of current observations $\{ O^1=o^1 \dots O^K=o^K \}$ the task is to choose an additional observation $O'$ among the set of observations $O_1 \dots O_N$ that maximizes mutual information between it and the hidden hypothesis. For brevity, we will write the past observations as $O^{-} := \{O^1=o^1 \dots O^K=o^K\}$. The objective is to find
\[
	O^{best} = \argmax_{O' \in \{O_1 \dots O_N\}} I(S, O' | O^{-}) ~~.
\]
Rather than computing the mutual information directly, the typical approach solves the equivalent problem of selecting the observation that minimizes posterior conditional entropy:
\begin{equation*}
\begin{split} 
\argmax_{O'} I(S, O' | O^{-}) = & \argmax_{O'} H(S | O^{-}) - H(S | O', O^{-}) \\
=& \argmin_{O'} H(S | O', O^{-}) ~~.
\end{split}
\end{equation*}

This objective of reducing posterior entropy on the hypothesis space is the approach used by Bellala (2013) and Zheng (2012). The posterior conditional entropy $H(S | O', O^{-})$ can be computed as:
\begin{align*}
 &H(S | O', O^{-}) \\
=& \sum_{o} P(O'=o | O^{-}) \big(\sum_{s}P(S=s|O'=o, O^{-}) 
 \\ & \log P(S=s|O'=o, O^{-}) \big) ~~.
\end{align*}
This computation requires a summation over a potentially exponentially large hypothesis space, leading to a time of $O(dom(O)dom(S))$. Zheng (2012) addresses this issue by taking advantage of the independence structures of the graph and develops efficient loopy belief propagation algorithms that approximate this entropy. Bellala (2013) develops a problem-specific ranking function that orders the observations according to their entropy. However, both approaches are problem specific and are difficult to generalize to other kinds of problems.

\subsection{APPROACH}

Rather than reducing the entropy on the hypothesis space, we consider an alternative decomposition:
\begin{equation*}
\begin{split}
I(S, O' | O^{-})=H(O' | O^{-}) - H(O' | S, O^{-}) ~~.
\end{split}
\end{equation*}
At first glance this decomposition is no less difficult: both entropy terms depend on $O'$, which would suggest that both need to be computed. We now introduce the assumption of uniform observation entropy, which allows us to simplify $H(O' | S, O^{-})$ to a constant.

\textbf{Definition} An observation problem $(S, O_1 \dots O_N)$ has the {\it uniform observation entropy} (UOE) property if, for each hypothesis, the conditional entropy for all observation dimensions $O_i$ are equal: $\forall s \in dom(S), ~~\forall i,j \in \{1 \dots N\}$:
\begin{equation*}
    H(O_i | S=s) = H(O_j | S=s) ~~.
\end{equation*}
This is a reasonable assumption for many domains: a common example is when the observations are obtained from a noisy channel that corrupts all observations equally. When observations are deterministic, the uniform observation entropy property holds trivially.

We begin by showing that the entropy of new observations remains uniform when conditioned on past observations:
\begin{lemma}
\begin{equation*}
\begin{split}
\forall i,j \in \{1 \dots N\}. ~~ H(O_i | S, O^{-}) = H(O_j | S, O^{-}) ~~.
\end{split}
\end{equation*}
\end{lemma}
\begin{proof}
By the definition of conditional entropy, we have
\begin{equation*}
\begin{split}
H(O_i | S, O^{-}) =& \sum_{s} P(S=s|O^{-}) H(O_i | S=s, O^{-}) \\
=& \sum_{s} P(S=s|O^{-}) H(O_i | S=s) \\
\intertext{Then applying UOE, we get}
=& \sum_{s} P(S=s|O^{-}) H(O_j | S=s) \\
=& H(O_j | S, O^{-}) ~~.
\end{split}
\end{equation*}
\end{proof}
We use this lemma to derive a much simpler form of our original problem of maximizing mutual information.
\begin{theorem} For an observation problem $(S, O_1 \dots O_N)$ with the UOE property:
\begin{equation*}
\begin{split}
\argmax_{O'} I(S, O' | O^{-}) = \argmax_{O'} H(O' | O^{-}) ~~.
\end{split}
\end{equation*}
\end{theorem}
\begin{proof}
According to the chain rule for mutual information, 
\begin{align*}
\argmax_{O'} &\,I(S, O' | O^{-}) \\
&=\argmax_{O'} H(O' | O^{-}) - H(O' | S, O^{-}) ~~. \\
\intertext{By the Lemma, the second term is constant for all possible values of $O'$, and hence}
&=\argmax_{O'} H(O' | O^{-}) ~~.
\end{align*}
\end{proof}
To compute $H(O' | O^{-})$ we require the probability values $P(O' | O^{-})$ for each possible $O'$. Naively, one can compute this probability by summing over the hypothesis space, but this would yield us no benefit. However, we can model this probability directly using a function approximator. The result is an efficient computation for $\argmax_{O'} H(O' | O^{-})$ that runs in $O(dom(O)NM)$ time, where $N$ is the total number of observations and $M$ depends on the complexity of our function approximator.

We call the model of $P(O' | O^{-})$ an \emph{implication model} because it attempts to deduce the outcome of a future observation based on a set of past observations. We will describe how to train this model in section 3.

\subsection{OBSERVATION COLLECTION}

We solve the active diagnosis problem in two parts: observation collection (OC) and hypothesis delivery. In OC, an observation is adaptively chosen at each time step by maximizing the entropy $H(O' | O^{-})$ until a budget of observations is exhausted. These observations are then used by the hypothesis delivery algorithm to obtain a hypothesis.

\begin{algorithm}
\kwRequire{$f(j,o,\lbrack o^1, \dots, o^K \rbrack)$ = $P(O_j=o| O^1=o^1 \dots O^K=o^K)$}
\kwRequire{$query_s(O)$}
 $Os$ = $\lbrack ~ \rbrack$ \\
 \While{size($Os$) $\le$ Budget}{
  O = $\argmax_{j \in 1 \dots N} H(O_j,Os)$ \\
  o = $query_s(O)$ \\
  $Os \leftarrow Os + \lbrack O=o \rbrack $ \\
 }
 \kwWhere{ $ H(O_j,Os) = \sum_{dom(O)} f(j,o,Os)\log(f(j,o,Os))$ }
 \kwReturn{$Os$}
 \caption{Observation Collection}
\end{algorithm}

The observation collection algorithm is detailed in Algorithm 1. Here, $f$ is a function that models $P(O_j=o|O^{-})$; $query$ is a function which returns the value for an observation given a hidden hypothesis $s$ that we are trying to discover. One can only call this function $Budget$ number of times.

Hypothesis delivery is done in a problem specific manner: For some problems it is appropriate to train a classifier that maps the collected observations to the hidden hypothesis; for other problems one can compute the hidden hypothesis directly from the observations. We will discuss the hypothesis delivery in detail in the experiment section. 

Our scheme is unique in that the observation collection algorithm is completely decoupled from the hypothesis delivery algorithm. This decoupling  allows one to explore different hypothesis delivery algorithms with the guarantee that the observations being used to deliver the hypothesis are well chosen. The observation collection algorithm depends only on $P(O' | O^{-})$, which can be trained on a dataset that is not annotated with hypothesis-specific information.

\section{IMPLICATION MODEL}

In this section we describe how to model $P(O' \mid O^{-})$. We first describe a neural network model that, given $O^{-}$, simultaneously computes $P(O'=o \mid O^{-})$ for all possible choices for $O'$ and $o$; We then describe how to train this neural network in a fashion similar to that of a sparse auto-encoder. Although this strategy can be generalized for non-binary discrete observations, we restrict our attention to the case of binary observations in the following.

\subsection{NEURAL NETWORK MODEL}

We consider a model that outputs $P(O' \mid O^{-})$ for all possible choices of observations simultaneously since the observation collection algorithm must determine the best $O'$ amongst all possible remaining observations. Note that $O^{-}$ has a variable length depending on how many observations have been made at the time of picking the next one. One may choose to encode this variable-length input with an LSTM, but we found that a simple feed-forward neural network that allows the value of an observation to be $\UNK$, the unknown value, to be more effective. 

Our model resembles a simple auto-encoder with a single, fully-connected hidden layer containing $M$ rectified linear units; however, the approach could be straightforwardly extended to use a deeper encoding/decoding network structure in more complex domains. The input to the network is a vector $(x^1_1,x^0_1,x^1_2,x^0_2,\dots,x^1_N,x^0_N)$, where:
\begin{equation*}
    (x^1_i,x^0_i) =
    \begin{cases}
      (1, 0) & \text{if}\ (O_i = 1) \in O^{-} \\
      (0, 1) & \text{if}\ (O_i = 0) \in O^{-} \\
      (0, 0) & \text{if}\ O_i = \UNK \\
    \end{cases} ~~.
\end{equation*}
The output is a vector $(y^1_1,y^0_1,y^1_2,y^0_2,\dots,y^1_N,y^0_N)$ where the distribution on values of $O_i$ is defined using softmax:
\begin{equation*}
    \begin{cases}
       P(O_i = 1) = & \frac{e^{y^1_i}}{e^{y^0_i}+e^{y^1_i}} \\
       P(O_i = 0) = & \frac{e^{y^0_i}}{e^{y^0_i}+e^{y^1_i}} \\
    \end{cases} ~~.
\end{equation*}
Figure~\ref{fig:imply} illustrates the architecture.

\begin{figure}[!ht]
  \centering
    \includegraphics[width=7cm,height=7cm,keepaspectratio]{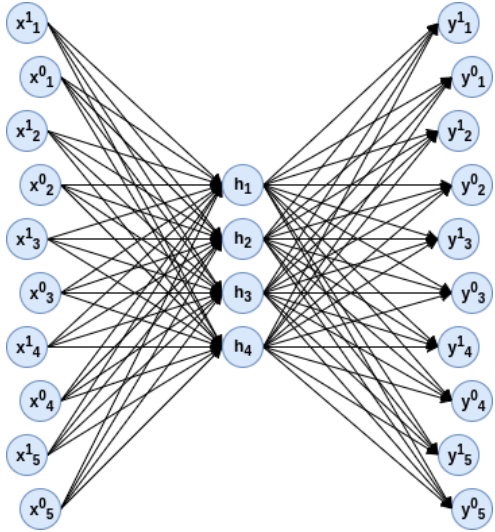}
  \caption{The implication model where there are a total of 5 observations and 4 hidden units}
    \label{fig:imply}
\end{figure}

\subsection{MODEL TRAINING}

Our training data is of the form\\ 
\[\{(x^{1(j)}_1,x^{0(j)}_1,x^{1(j)}_2,x^{0(j)}_2,\dots,x^{1(j)}_N,x^{0(j)}_N\}_{j=1 \dots n}\] where $(x^{1(j)}_i,x^{0(j)}_i)$ corresponds to the value of observation $O_i$ made on hidden hypothesis $s^{(j)} \sim P(S)$. Note that the value of $s^{(j)}$ need not be present in the data.

We perform supervised training of the network using stochastic gradient descent with input-output pairs $(\hat{x}^{(j)}, x^{(j)})$. We draw $x^{(j)}$ randomly from the training data set, and construct $\hat{x}^{(j)}$  from $x^{(j)}$ by randomly changing some elements of $x^{(j)}$ to $\UNK$ as follows: We first draw a number $\epsilon \sim uniform(0,1)$, then change each $x^{(j)}_i$ to $\UNK$ with probability $\epsilon$.  We use cross-entropy as the loss function between the softmax probability of $O_i$ and the expected output $(x^{1(j)}_i,x^{0(j)}_i)$. For the experiments in this paper, the hidden layer contains 200 units and is equipped with the ReLU activation function. We implement and train the implication model using TensorFlow.

One can think of this model as a type of sparse autoencoder Le (2013), but rather than sparsifying the hidden layer we are sparsifying the input layer. Intuitively, without any sparsification, the network learns the identity implication, $O_i=o \Rightarrow O_i=o$. Sparsifying the input layer forces the network to learn inter-relationships between the observations, such that a handful of observations is sufficient to recover the full set of observations.

Note that by randomly ablating a subset of the observation features, we learn a model that is capable of working with any subset of observations $O^{-}$. This could potentially be more work than necessary, as during observation collection we are following a particular greedy policy that chooses the new observation with the maximum entropy, leading to particular sequences of observations being made, which restricts the set of queries. A good avenue for future work might be to train on observation traces generated by OC so the network can specialize further to a particular distribution of observations.

\section{EXPERIMENTS}

In this section we evaluate the performance of our active diagnosis algorithm on several simple active diagnosis problems. Here we outline the problems and their characteristics before elaborating them in detail \footnote{For specifics of these experiments see: \url{https://github.com/evanthebouncy/uai2017_learning_to_acquire_information}}.

\paragraph{Battleship} This is a variant of the classic battleship game, where the goal is to infer the configurations of all 5 ships on a 10 by 10 board with as few queries to the board as possible. The hypothesis space consists of approximately $2^{38}$ different configurations, and each query is a (x,y) coordinate that results in ``hit'' or ``miss''.

\paragraph{Sushi} This is a problem of preference elicitation, where the goal is to learn the full preference order of an unseen user on 10 different types of sushi with as few pair-wise comparison queries as possible. The hypothesis space consists of $10!$ potential full rankings, and each query is a pair-wise comparison between 2 sushi types that results in ``true'' or ``false''.


\paragraph{Network} We consider a fault localization task on a network of 100 nodes organized as a tree, where each link has 2\% chance of failure. The task is to learn an efficient scheme for querying pair-wise connectivities to localize the failure with as few queries as possible. The hypothesis space consists of all $2^{100}$ combinations of failures. The query is a pair-wise connectivity check, restricted to all 99 direct link checks and 300 additional pairs of fixed nodes with measurement equipment.

\subsection{BATTLESHIP}

The task is to infer the locations of all the ships with as few queries as possible. The board is a 10 by 10 grid, with 5 ships of size $2\times4, 1\times5, 1\times3, 1\times3, 1\times3$ placed randomly with arbitrary horizontal or vertical orientations. Adjacent placements are allowed, but the ships may not overlap with each other. See Figure \ref{fig:battleship_belief} for an example board.

At each stage of the game, our observation collection algorithm OC selects the most-informative observation and updates its belief space. Figure \ref{fig:battleship_belief} illustrates the belief space at various numbers of observations. Note that without any observations, our model predicts that a ship is more likely to be located near the center of the board rather than toward the edge. In the case of 22 observations, our model queried a ``hit'' on the lower left without completely observing the $1\times5$ long ship in the middle of the board.

\begin{figure}[!ht]
  \centering
    \includegraphics[width=8.2cm,height=7cm,keepaspectratio]{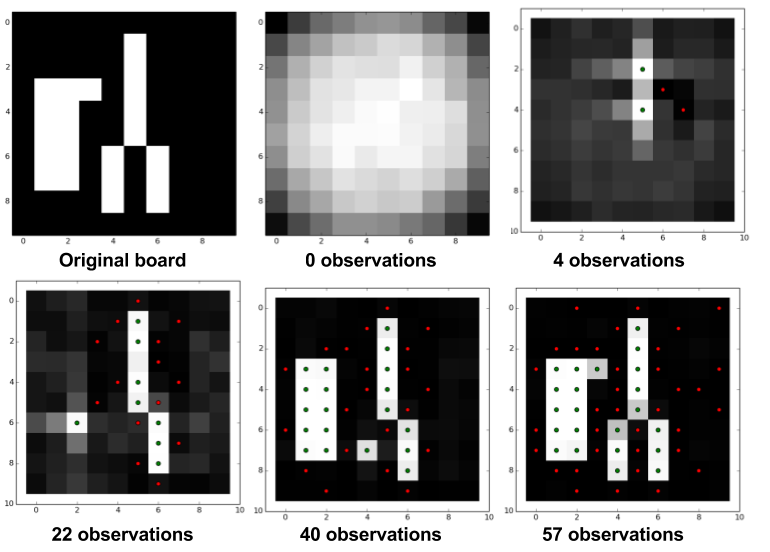}
  \caption{Belief space of observations on a particular board at various numbers of observations. Intensity indicates the probability of a coordinate being a hit, and colored dots indicates past observations: green for ``hit'' and red for ``miss''}
    \label{fig:battleship_belief}
\end{figure}

We consider two kinds of hypothesis delivery schemes. In the first scheme we deliver the probabilities of all the observations using the implication model $P(O'|O^{-})$, implicitly inferring the hidden locations of the ships. In the second scheme we use a constraint solver which takes in the list of observations made by OC and returns a hypothesis $s$ that is consistent with these observations.

For comparison we consider 2 baseline algorithms: The random sampler \textbf{rand} samples unseen coordinates at random. The \textbf{sink} algorithm samples unseen coordinates at random, and when it has found a hit, it queries all its neighboring coordinates until no ``hit'' coordinates can be found, then resumes random sampling. Both the random and sink algorithms initially mark all coordinates as ``miss'' and update them to ``hit'' when a hit has been recorded. 

To evaluate performance under the first hypothesis delivery scheme, we let OC output the maximum likelihood guess for each coordinate. Accuracy is measured as a fraction of correctly guessed coordinates. Figure \ref{fig:battleship_pred} compares accuracy across all coordinates as a function of number of observations; an accuracy of 1.0 means all coordinates are guessed correctly. In this experiment, we consider 3 different variant of the implication model: \textbf{oc\_1} is the single hidden-layer fully-connected model originally described, \textbf{oc\_0} is the 0-layer neural-network model (logistic regression), and \textbf{oc\_cnn} contains a convolutional neural-network layer before a fully-connected hidden layer. As we can see, the random algorithm improves linearly as expected, the sink heuristic performs better than random, and our approach OC performs the best, with \textbf{oc\_1} and \textbf{oc\_cnn} performing better than the logistic regression \textbf{oc\_0}. Figure \ref{fig:battleship_pred_noise} considers the same experiment except a 10\% observation error is introduced, under this condition all 3 variants of the OC similarly, and more robust to noise than the baseline algorithms. 

\begin{figure}[!ht]
  \centering
    \includegraphics[width=8.2cm,height=7cm,keepaspectratio]{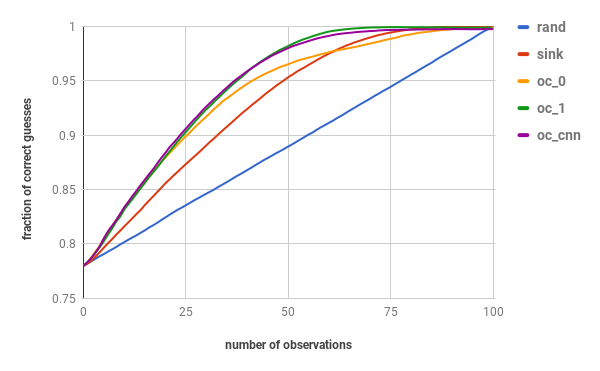}
  \caption{Comparison of our algorithm \textbf{oc} against the 2 baselines. Accuracies are averaged over 1000 randomly generated boards}
    \label{fig:battleship_pred}
\end{figure}

\begin{figure}[!ht]
  \centering
    \includegraphics[width=8.2cm,height=7cm,keepaspectratio]{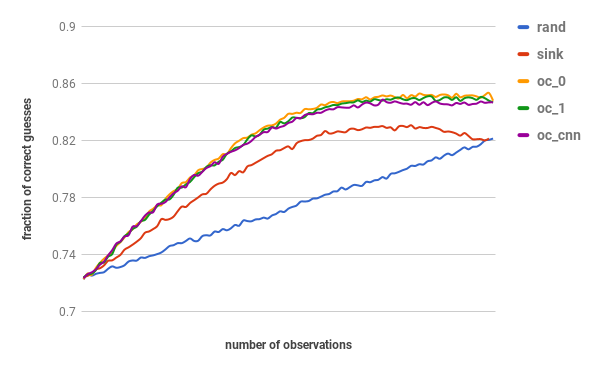}
  \caption{Comparison of our algorithm \textbf{oc} against the 2 baselines. Accuracies are averaged over 1000 randomly generated boards with 10\% chance of observation error}
    \label{fig:battleship_pred_noise}
\end{figure}

To evaluate the performance under the second hypothesis delivery scheme, we use a constraint solver to produce a hypothesis in the form of ships' locations and orientations, using observations collected by \textbf{rand}, \textbf{sink} and \textbf{oc} as constraints to the hypothesis. We then measure accuracy by the number of correctly predicted ship's locations and orientations: 5 means all 5 ship's locations and orientations are correctly produced by the constraint solver given the observations collected. Figure \ref{fig:battleship_solver} compares number of correctly guessed ships as a function of number of observations. As we can see our algorithm OC performs the best, followed by sink with random performing the worst. Note that constraint solvers are known to produce arbitrary hypotheses that satisfy the constraint in an under constrained system, yet despite this, OC was able to consistently out perform the baseline algorithms.

\begin{figure}[!ht]
  \centering
    \includegraphics[width=8.2cm,height=7cm,keepaspectratio]{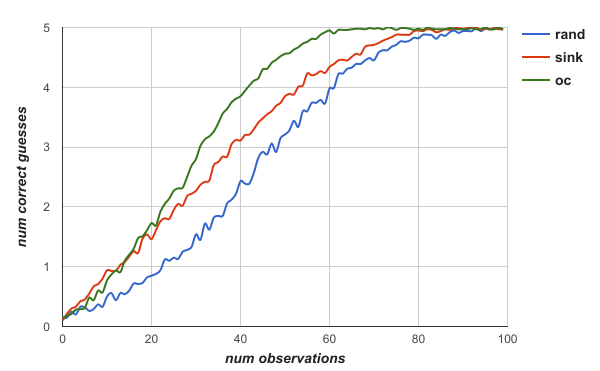}
  \caption{Comparison of our algorithm \textbf{oc} against the 2 baselines when using a constraint solver for hypothesis delivery. Accuracies are averaged over 100 randomly generated boards. Only the single-hidden-layer implication model is considered here.}
    \label{fig:battleship_solver}
\end{figure}

\subsection{SUSHI}

The sushi data set collected by Kamishima (2003) contains 5000 user preferences for 10 kinds of sushi expressed in \emph{full rankings}. The dataset also contains feature vectors describing each individual type of sushi and describing the users, which in this study we omit. Our task is the following: Given a new user, how to infer the full ranking of this user with as few pair-wise comparison queries as possible? Naively, this is a sorting problem where one can obtain the full ranking of any permutation of items with $O(n\log n)$ comparisons. However, the preference orderings are not uniformly random, for instance, a preference of eel over tuna may indicate a user's liking of cooked sushi over raw sushi. We evaluate accuracy by using the Kendall correlation: a value of 1.0 means all pair-wise orderings of our prediction and the ground truth agree with each other, and a value of -1.0 indicates all pair-wise orderings are in disagreement.

There is no related work on this dataset that attempts to discover the full preference of a new user based on pair-wise queries to the new user. Soufiani (2013) models the sushi preferences as a generative process where the preferences are caused by a combination of user features (such as age and gender) and sushi features (such as price) but does not perform elicitation on a new user in the form of pair-wise queries. Without further elicitation, they are able to infer a new user's preference with a Kendall correlation of 0.75. Guo (2010) performs elicitation on a new user in the form of pair-wise queries but instead of learning the full preference, attempts to recommend the best sushi to the user. They are able to always predict the best sushi after 14 pair-wise comparisons. Thus, for comparison we consider various in-place sorting algorithms as baseline. Each time a pair-wise comparison question is made, we take a snapshot of the current array and extract pair-wise orderings from it.

The 5000 user preferences are split into a 2500 preferences training set and 2500 preferences testing set. In order to train our implication model to handle novel permutations not seen during the training set, we augment the training set by a set of randomly sampled permutations in addition to the 2500 preferences.

For this experiment, we can measure performance directly as Kendall correlation without delivering an explicit ordering as the hypothesis (If one wishes one can easily compute a full ordering from all pair-wise orderings). Figure \ref{fig:sushi_pred} compares our observation collection algorithm \textbf{oc\_0} and \textbf{oc\_1} against various in-place sorting algorithms: BubbleSort \textbf{bsort}, QuickSort \textbf{qsort}, and MergeSort \textbf{msort}. As we can see, even without any observations OC was able to obtain a Kendall correlation of 0.3, indicating the underlying distribution for preferences is not uniform; by comparison, the baseline sorting algorithms make no assumptions on the underlying distribution, thus starts with a correlation around 0. Our scheme was able to infer the full ranking of any user in 26 queries, beating the performance of \textbf{qsort}, which takes 40 queries. 

\begin{figure}[!ht]
  \centering
    \includegraphics[width=8.2cm,height=7cm,keepaspectratio]{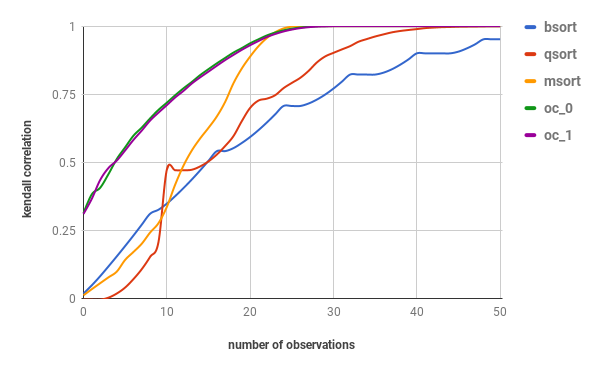}
  \caption{Kendall correlation as a function of number of queries averaged across 2500 testing examples}
    \label{fig:sushi_pred}
\end{figure}

We also considered the case where the query has a 10\% chance of error, the result is shown in Figure \ref{fig:sushi_pred_noise}. Note the OC algorithm is robust in the presence of noise as it improves its pair-wise preferences incrementally with each observation (a property also shared by \textbf{bsort}) whereas deterministic baseline such as \textbf{qsort} and \textbf{msort} sorts the preferences assuming all observations are perfect. 

\begin{figure}[!ht]
  \centering
    \includegraphics[width=8.2cm,height=7cm,keepaspectratio]{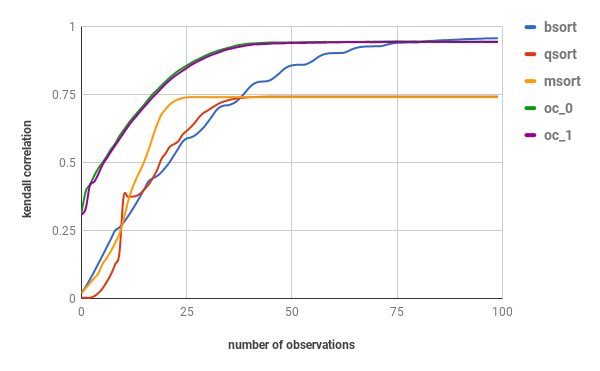}
  \caption{Kendall correlation as a function of number of queries averaged across 2500 testing examples, where each query has a 10\% chance of error}
    \label{fig:sushi_pred_noise}
\end{figure}

\subsection{NETWORK}

We consider the task of fault localization on a network of nodes organized as a tree. The network without any failure is shown in Figure \ref{fig:network}. There are 100 nodes in this network with 99 direct links, forming a spanning tree. The hypothesis space consists of failure cases where each direct link has a probability of 2\% of failure. The query is a pair-wise connectivity check, restricted to all 99 direct link checks and 300 additional pairs of fixed nodes. 

\begin{figure}[!ht]
  \centering
    \includegraphics[width=8.2cm,height=7cm,keepaspectratio]{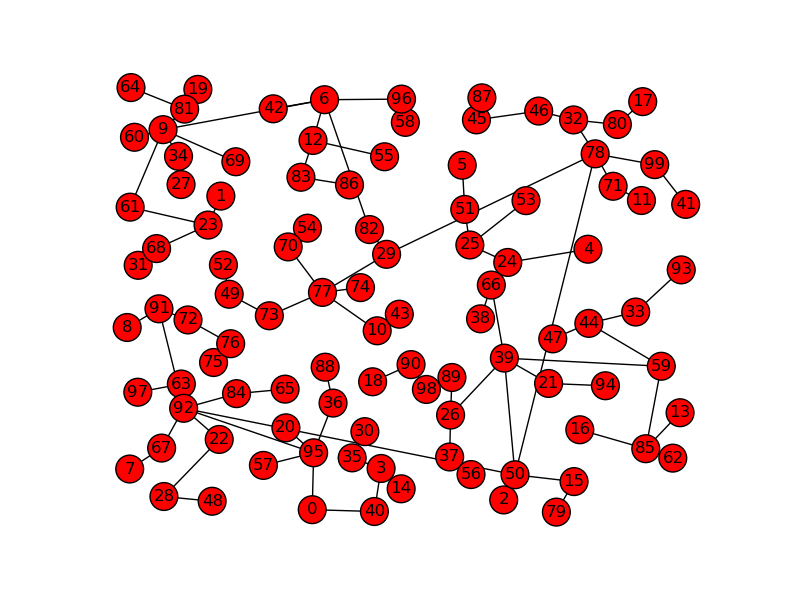}
  \caption{The network without any failure}
    \label{fig:network}
\end{figure}

For hypothesis delivery, we use the implication model $P(O'|O^{-})$ to output the probability of failure on the 99 directly connected pairs of nodes. To measure accuracy, we use the maximum likelihood guess for these pairs and measure the fraction of correctly diagnosed link failures.

We compare our approach with the naive localization scheme \textbf{rand} that randomly picks an unobserved direct link and checks if it is disconnected. The random scheme will always be able to diagnose all the link failures in 99 observations. However, OC can leverage the network structure to learn an efficient querying scheme by utilizing connectivity queries to the additional pairs of nodes. Figure \ref{fig:network_pred} shows accuracy of OC against the random algorithm. We see OC outperforms the random algorithm most of the time, reaching 99.5\% accuracy in 60 observations. OC performs worse than the random scheme for less than 10 observations because OC does not query the direct links for the first few observations: The 300 additional pairs can yield more information without pin-pointing an exact failure. OC again performs worse than the random scheme past 80 observations. This is a consequence of the greedy approach: OC maximizes information gain one observation at a time, but the set of observations selected by the random approach (all direct links) is the optimal set of size 99, capable of answering any failure queries. OC fails to find an alternative optimal set partly because there are instances where many links fail simultaneously, which weakens the structures that our implication model depends on.

\begin{figure}[!ht]
  \centering
    \includegraphics[width=8.2cm,height=7cm,keepaspectratio]{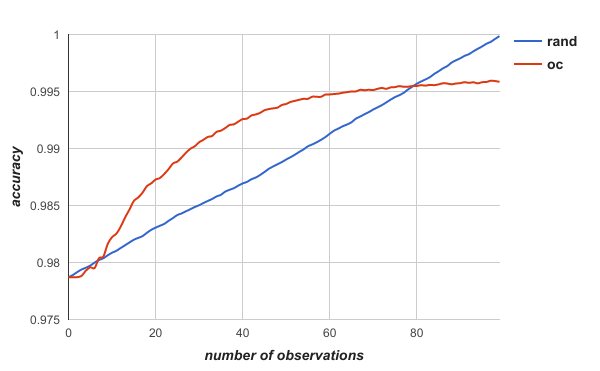}
  \caption{Accuracy of link failure diagnosis averaged across 1400 random instances}
    \label{fig:network_pred}
\end{figure}

To quantify the effects of structure on the difficulty to diagnose a link failure, we define the \emph{dependency coefficient} for a direct link $e_i$ as: 

\begin{equation*}
	dep_i = \frac{M_i N_i}{(M_i + N_i)(M_i + N_i - 1)}
\end{equation*}

Where $M_i$ and $N_i$ are the sizes of the disconnected subgraphs when only link $e_i$ fails. The dependency coefficient measures the fraction of pair-wise connectivity which would be affected should $e_i$ fail. When $e_i$ is located next to a leaf node in the network, this coefficient is close to 0.01. When disconnecting $e_i$ breaks the network into 2 equally sized components, the coefficient is close to 0.25.

Figure \ref{fig:dependency_coef} measures accuracy of diagnosis against dependency coefficient. As one can see, in general higher dependency coefficient means better diagnosis: When many pair-wise connectivities depend on a link, successful pair-wise connection on any of these pairs will imply the link has not failed; conversely, when this link fails all pair-wise connections that depend on this link fail as well. This property is useful because a link which many pair-wise connections depend on would have a better chance of being correctly diagnosed.
\begin{figure}[!ht]
  \centering
    \includegraphics[width=8.2cm,height=7cm,keepaspectratio]{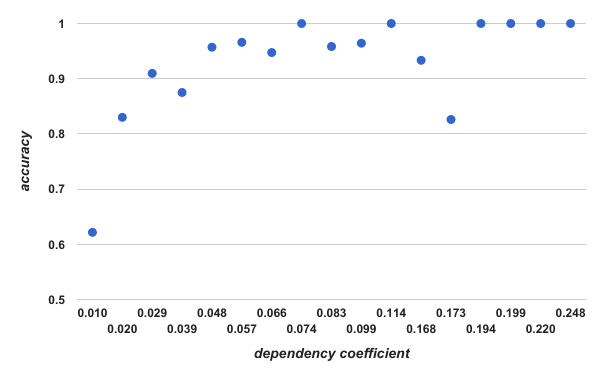}
  \caption{Accuracy of link failure diagnosis as a function of dependency coefficients, averaged over 1400 instances }
    \label{fig:dependency_coef}
\end{figure}

\section{DISCUSSION}

In this paper we present a new approach to active diagnosis in domains with complex hypothesis spaces. We show that given the uniform observation entropy assumption, the problem of choosing the most informative query reduces to maximizing the conditional observation entropy $H(O'|O^{-})$, which can be computed from $P(O'|O^{-})$ without explicit modeling of the hypothesis space. We describe a neural network modeling of $P(O'|O^{-})$ called an implication model, and demonstrate that it can be trained in a supervised fashion. We evaluate our approach on several simple yet distinct benchmarks, demonstrating superior performance against reasonable baseline algorithms.

As our proof and experiments depends on the UOE assumption, the applicability of the UOE assumption is a legitimate concern. While it is true that UOE assumption is not valid under all diagnostic problems, UOE does apply to the class of problems where the same kind of sensor is used to gather signals at different locations of a deterministic environment, including robotic localization, geostatistics, and optimal sensor placements.

A unique feature of our approach is that the implication model can be trained on unlabeled observation data in a supervised fashion. A good direction for future work would be to apply this approach to domains where there are an abundance of observation data without annotations to the hypothesis that generated it. The implication model learns interesting structures of the observation data; it would be instructive to quantify how easy it is to learn different kinds of structures.

\subsubsection*{Acknowledgements}
We would like to thank Rohit Singh for the encoding of the BattleShip SAT formulation, Shachar Itzhaky for insightful discussions, and Twitch Chat for moral supports.

This work was funded by the MUSE program (Darpa grant
FA8750-14-2-0242). 


\subsubsection*{References}


Bellala, Gowtham, et al. "A rank-based approach to active diagnosis." IEEE transactions on pattern analysis and machine intelligence 35.9 (2013): 2078-2090.

Golovin, Daniel, and Andreas Krause. "Adaptive submodularity: Theory and applications in active learning and stochastic optimization." Journal of Artificial Intelligence Research 42 (2011): 427-486.

Guo, Shengbo, Scott Sanner, and Edwin V. Bonilla. "Gaussian process preference elicitation." Advances in Neural Information Processing Systems. 2010.

Holub, Alex, Pietro Perona, and Michael C. Burl. "Entropy-based active learning for object recognition." Computer Vision and Pattern Recognition Workshops, 2008. CVPRW'08. IEEE Computer Society Conference on. IEEE, 2008.

Joshi, Ajay J., Fatih Porikli, and Nikolaos Papanikolopoulos. "Multi-class active learning for image classification." Computer Vision and Pattern Recognition, 2009. CVPR 2009. IEEE Conference on. IEEE, 2009.

Krause, Andreas, and Carlos Guestrin. "Submodularity and its applications in optimized information gathering." ACM Transactions on Intelligent Systems and Technology (TIST) 2.4 (2011): 32.

Le, Quoc V. "Building high-level features using large scale unsupervised learning." Acoustics, Speech and Signal Processing (ICASSP), 2013 IEEE International Conference on. IEEE, 2013.

Lewis, David D., and William A. Gale. "A sequential algorithm for training text classifiers." Proceedings of the 17th annual international ACM SIGIR conference on Research and development in information retrieval. Springer-Verlag New York, Inc., 1994.

Littman, Michael L., Judy Goldsmith, and Martin Mundhenk. "The computational complexity of probabilistic planning." Journal of Artificial Intelligence Research 9.1 (1998): 1-36.

Mnih, Volodymyr, Nicolas Heess, and Alex Graves. "Recurrent models of visual attention." Advances in neural information processing systems. 2014.

Nowak, Robert. "Generalized binary search." Communication, Control, and Computing, 2008 46th Annual Allerton Conference on. IEEE, 2008.

Rish, Irina, et al. "Real-time problem determination in distributed systems using active probing." Network Operations and Management Symposium, 2004. NOMS 2004. IEEE/IFIP. Vol. 1. IEEE, 2004.

Seung, H. Sebastian, Manfred Opper, and Haim Sompolinsky. "Query by committee." Proceedings of the fifth annual workshop on Computational learning theory. ACM, 1992.

Soufiani, Hossein Azari, David C. Parkes, and Lirong Xia. "Preference elicitation for general random utility models." arXiv preprint arXiv:1309.6864 (2013).

Srivastava, Nitish, et al. "Dropout: a simple way to prevent neural networks from overfitting." Journal of Machine Learning Research 15.1 (2014): 1929-1958.

T. Kamishima, "Nantonac Collaborative Filtering: Recommendation Based on Order Responses", KDD2003, pp.583-588 (2003)

Tong, Simon, and Daphne Koller. "Support vector machine active learning with applications to text classification." Journal of machine learning research 2.Nov (2001): 45-66.

Tong, Simon, and Edward Chang. "Support vector machine active learning for image retrieval." Proceedings of the ninth ACM international conference on Multimedia. ACM, 2001.

Wikipedia contributors. "Battleship (game)." Wikipedia, The Free Encyclopedia. Wikipedia, The Free Encyclopedia, 11 Mar. 2017. Web. 26 Mar. 2017.

Zheng, Alice X., Irina Rish, and Alina Beygelzimer. "Efficient test selection in active diagnosis via entropy approximation." arXiv preprint arXiv:1207.1418 (2012).

\end{document}